\documentclass[11pt]{article}

\usepackage{acl}

\usepackage{times}
\usepackage{booktabs}
\usepackage{balance}
\usepackage{multirow}
\usepackage{colortbl}
\usepackage{xcolor}
\usepackage{latexsym}
\usepackage{booktabs} 
\usepackage{algorithm}
\usepackage{amsmath}    
\usepackage{amssymb}    
\usepackage{algorithm}
\usepackage{algpseudocode}
\usepackage{amsthm}

\newtheorem{lemma}{Lemma}[section]

\usepackage{booktabs}
\usepackage{xcolor}
\usepackage{colortbl} 
\definecolor{firstplace}{HTML}{FFF2CC}  
\definecolor{secondplace}{HTML}{EFEFEF} 

\usepackage[T1]{fontenc}

\usepackage[utf8]{inputenc}

\usepackage{microtype}

\usepackage{inconsolata}

\usepackage{graphicx}

\title{\textsc{SkillC}: Learning Autonomous Skill Internalization in LLM Agents via Contrastive Credit Assignment}

\author{Hongxiang Lin, Zhirui Kuai, Erpeng Xue\thanks{Corresponding author.}, Lei Wang \\
  Meituan\\
\texttt{linhx0@hotmail.com} \quad \texttt{\{kuaizhirui, wanglei46, xueerpeng\}@meituan.com}}

\begin{document}
\maketitle

\begin{abstract}


Structured skill prompts improve exploration in long-horizon agentic reinforcement learning (RL). Skill-augmented RL methods retain external skills at inference, while skill-internalization RL methods withdraw them during training to enable autonomous performance. However, existing internalization approaches only use skill-helpfulness contrast for curriculum control, leaving the policy update unchanged and unable to distinguish skill-dependent from autonomous success.
We propose \textsc{SkillC}, a framework based on Contrastive Skill Credit Assignment (CSCA) that converts this contrast into a direct learning signal for internalization. \textsc{SkillC} samples paired skill-injected and skill-free rollouts for tasks from active skill types within the same policy update, and injects their task-level contrast into optimization via a dual-stream advantage estimator that preserves global ranking while applying a one-sided correction toward skill-free success. A smoothed validation-level signal further drives an adaptive curriculum over attribution strength, rollout allocation, and monotonic active-set pruning.
Experiments on ALFWorld and WebShop show that, without runtime skill access, \textsc{SkillC} surpasses the strongest prior skill-internalization RL baseline by 5.5\% and 4.4\%, respectively, while remaining competitive with skill-augmented RL methods.

\end{abstract}

\section{Introduction}

\begin{figure}[!t]
    \centering
    \includegraphics[width=\columnwidth]{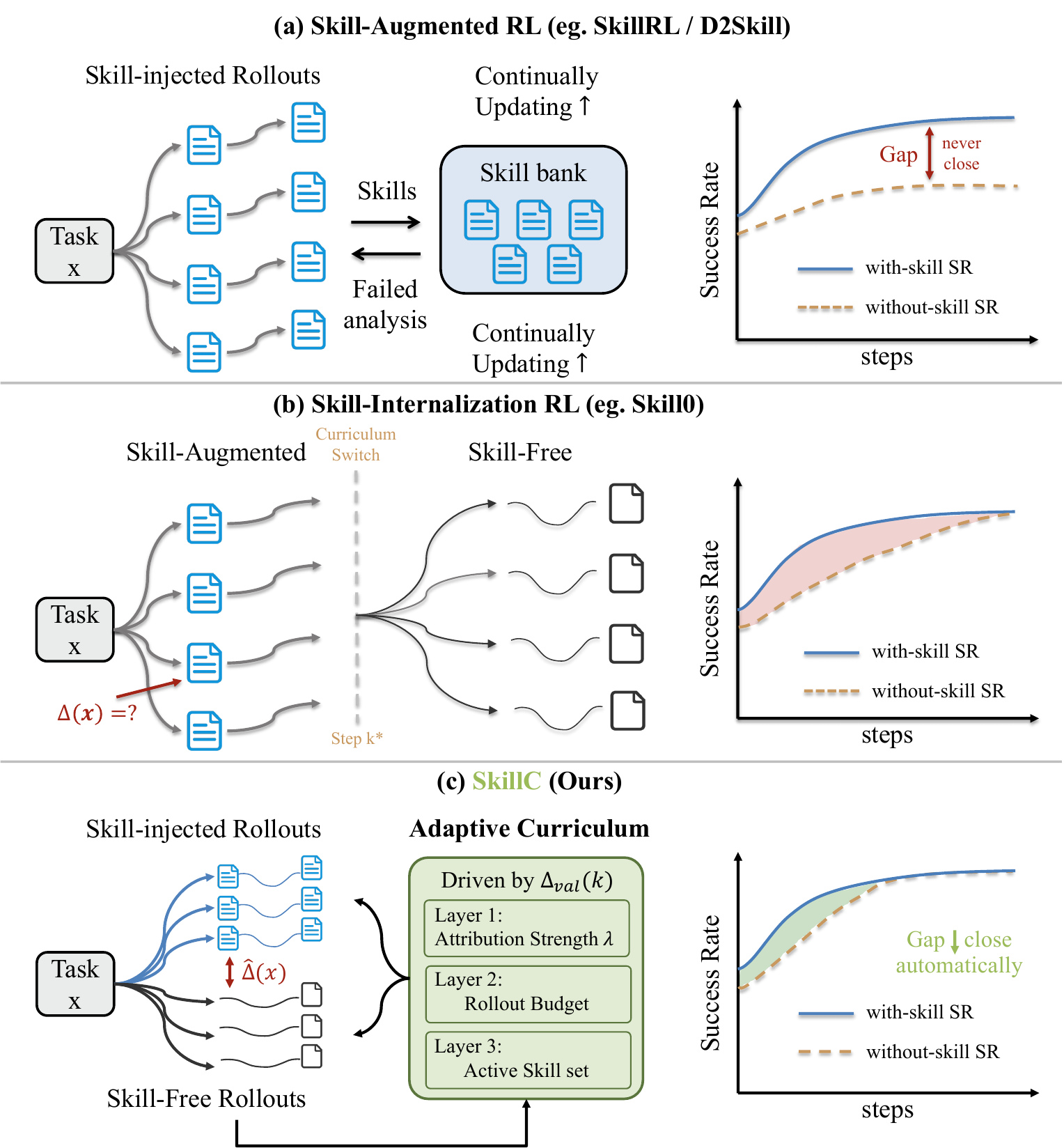}
    \caption{
        Comparison of paradigms for skill use in Agentic RL. (a) Skill-augmented RL methods keep skills available at inference time and optimize performance with runtime skill support. (b) Skill-internalization RL methods withdraw skills during training and may estimate skill helpfulness for control, but leave task-level credit assignment unchanged. (c) \textsc{SkillC} turns helpfulness comparison into contrastive credit assignment for autonomous success, while adapting attribution strength, rollout allocation, and the active skill set through a validation-driven curriculum.
    }
    \label{fig:paradigms}
\end{figure}

Structured skill prompts \cite{xu2026agent,zhang2026memskill,liu2026well} have become a practical mechanism for training long-horizon LLM agents. By providing reusable procedural guidance during interaction, they improve exploration and sample efficiency in agentic reinforcement learning (RL) \cite{li2026organizing,li2026arise,su2026skill,cho2026skillret}. Recent work has advanced this paradigm through skill augmentation and co-evolving skill banks \cite{li2026arise}, skill distillation \cite{wang2026skill,xia2026skillrl}, and skill lifecycle management \cite{pu2026skillops,ouyang2026skillos,shen2026dynamic}. Yet most of these approaches remain in the skill-augmented regime shown in Figure~\ref{fig:paradigms}(a): skills are treated as external runtime support, and autonomous performance without retrieval is not the optimization target.

A more ambitious goal is skill internalization, illustrated in Figure~\ref{fig:paradigms}(b), where skills are used only during training and gradually withdrawn so that the final policy can act autonomously. \textsc{Skill0}~\cite{lu2026skill0} is a representative example. Current internalization methods can already estimate whether a skill remains useful and use that information to control withdrawal. What remains unchanged, however, is the policy objective. Within each GRPO update \cite{shao2024deepseekmath}, all rollouts for a task are still generated under a single prompt condition, so optimization reinforces success under the current regime rather than autonomy after withdrawal. The training loop may know that a skill is still globally useful, yet it cannot determine whether a particular successful trajectory is already independent of that skill.

We term this optimization mismatch \emph{internalization blindness}. To address it, we propose \textsc{SkillC}, a framework based on Contrastive Skill Credit Assignment (CSCA) that converts this contrast into a direct learning signal for internalization. For tasks from active skill types, \textsc{SkillC} samples paired skill-injected and skill-free rollouts within the same policy update, injecting their task-level contrast into optimization via a dual-stream advantage estimator that preserves global ranking while applying one-sided correction toward skill-free success. A smoothed validation-level counterpart of the same contrastive quantity further drives a hybrid adaptive curriculum. For multi-category environments like ALFWorld, this curriculum jointly controls attribution strength, rollout allocation, and monotone active-set pruning; for single-skill environments like WebShop, it controls attribution strength and rollout allocation uniformly. Together, these mechanisms yield the contrastive internalization regime illustrated in Figure~\ref{fig:paradigms}(c).

Our main contributions are as follows:
\begin{itemize}
    \item We identify trajectory-level credit assignment limitations in fixed-schedule skill internalization where all rollouts share a single prompt condition, preventing the gradient from distinguishing autonomous versus skill-dependent success. We propose \textsc{SkillC} with paired contrastive rollouts and adaptive curriculum scheduling to address this.

    \item We implement CSCA in \textsc{SkillC} through three coordinated mechanisms. Paired skill-injected and skill-free rollouts expose task-level internalization gaps, a dual-stream advantage estimator with per-condition normalization redirects credit toward autonomous success, and a validation-driven curriculum adapts attribution strength and rollout allocation.

    \item Critical characterization of applicability reveals CSCA's effectiveness is signal-quality-dependent. Strong contrastive signals enable substantial gains (Heat $+10.5\%$, Cool $+11.1\%$, Pick2 $+33.4\%$), while weak signals on ceiling-effect tasks cause regression (Pick $-11.5\%$), and single-skill environments yield only marginal improvements (WebShop $+3.1\%$).
    
\end{itemize}
\section{Related Work}

\subsection{Agentic Reinforcement Learning}
Reinforcement learning has become a key post-training paradigm for language model agents. GRPO \cite{shao2024deepseekmath} introduced critic-free policy optimization using group-relative advantages. Building on this foundation, agentic RL methods address agent-specific challenges through hierarchical RL for long-horizon credit assignment \cite{li2026arise}, self-evolution mechanisms for policy co-adaptation \cite{fan2026evolving, shi2026skill1}, experience reuse for continuous adaptation \cite{chhikara2025mem0, liu2026simplemem}, and prompt-based reasoning like ReAct \cite{yao2022react} and Reflexion \cite{shinn2023reflexion}. However, a key challenge remains: how to leverage intermediate guidance during training while ensuring the final policy acquires genuine autonomous competence rather than guidance-dependent behaviors.

\subsection{Skill Libraries for Agents}

Structured skills have become a core mechanism for extending agent capabilities beyond parametric memory. Early work recognized that raw trajectory retrieval incurs substantial token overhead and noise, motivating abstraction into reusable procedural knowledge \cite{xia2026skillrl, lu2026skill0}. Recent skill-based systems span multiple dimensions: hierarchical skill organization and discovery through reflection \cite{tu2026dynamic, wang2026skillx}, dependency-aware retrieval via graph structures \cite{li2026graph}, and dynamic lifecycle management including retention and retirement \cite{shen2026dynamic,pu2026skillops}. For skill selection, approaches range from frozen embedding-based retrieval to learned routing and policy-driven ranking \cite{xia2026skillrl, cho2026skillret}. For utilization, RL-based methods condition policies on selected skills with hierarchical reward structures \cite{tu2026dynamic, shi2026skill1}. For distillation, methods span prompt-based extraction and self-reflection to co-evolving skill generators \cite{xia2026skillrl, wang2026skill}. However, existing skill-augmented RL methods do not jointly optimize all three stages with unified learning signals, and existing skill internalization approaches progressively withdraw skills according to curriculum schedules but leave the policy update itself unchanged during training. Within each gradient step, all rollouts share the same prompt condition, preventing the gradient from observing which behaviors are already autonomous versus skill-dependent. This creates an optimization mismatch we term internalization blindness.

\section{Method}

\begin{figure*}[t]
    \centering
    \includegraphics[width=\textwidth]{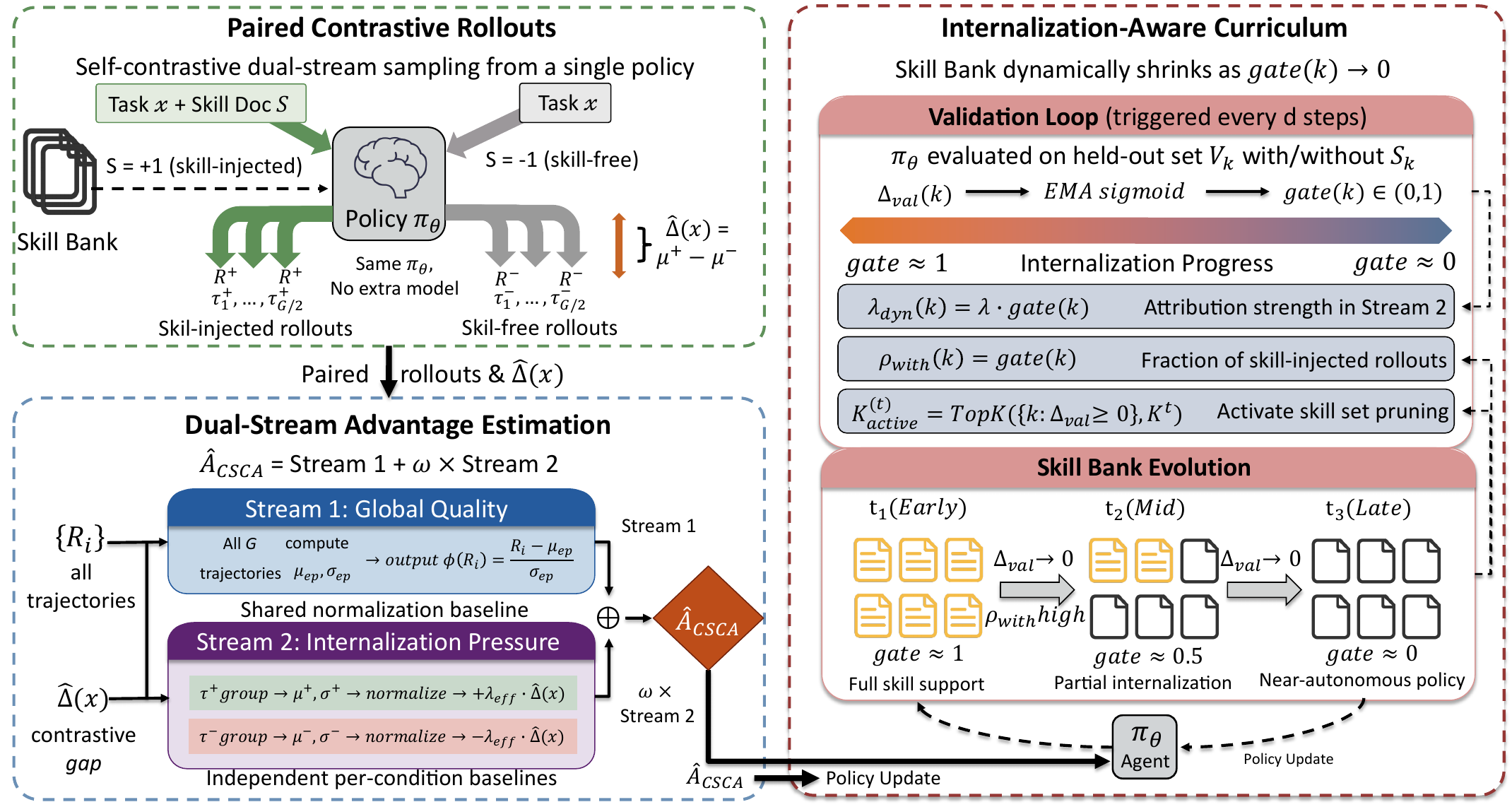}
    \caption{
    Overview of \textsc{SkillC}. Paired contrastive rollouts expose residual skill dependence for each task, a dual-stream advantage estimator redirects credit toward autonomous success without mixed-normalization bias, and a internalization-aware  curriculum adapts attribution strength, rollout allocation, and active skill set.
    }
    \label{fig:method}
\end{figure*}

We propose \textsc{SkillC}, addressing a key limitation in skill internalization: within each policy update, standard GRPO jointly normalizes skill-injected and skill-free trajectories, inflating the reward baseline when skill-injected rollouts dominate and suppressing credit for autonomous success. \textsc{SkillC} resolves this through three mechanisms: (1) sampling paired skill-injected and skill-free rollouts to measure task-level internalization gaps, (2) using per-condition advantage normalization to redirect credit toward autonomous behaviors, and (3) adapting the curriculum via validation-level contrastive signals to balance attribution strength, rollout allocation, and skill retention.

\subsection{Problem Formulation}
\label{sec:formulation}

For a task $x \sim \mathcal{X}$, we maintain a fixed skill bank $\mathcal{B}=\{\mathcal{S}_k\}_k$ and an active task-type set $\mathcal{K}_{\mathrm{active}}^{(t)}$. Let $z\in\{0,1\}$ denote the prompt condition: $z=1$ for skill-injected rollouts and $z=0$ for skill-free. For active task types $k(x)\in\mathcal{K}_{\mathrm{active}}^{(t)}$, we sample paired rollouts $\tau^+\sim\pi_\theta(\cdot\mid x,\mathcal{S}_{k(x)})$ and $\tau^-\sim\pi_\theta(\cdot\mid x)$. For retired types, only $\tau^-$ is sampled. Both use the same policy parameters, yielding self-contrastive design.

\paragraph{Task-Level Contrastive Gap.}
Unlike skill-augmented methods, \textsc{SkillC} targets without-skill success at inference. The task-level contrastive gap is
\begin{equation}
  \Delta(x)
  =
  \mathbb{E}[R \mid x, z=1;\pi_\theta]
  -
  \mathbb{E}[R \mid x, z=0;\pi_\theta],
\end{equation}
where $z\in\{0,1\}$ indicates skill injection. Positive $\Delta(x)$ indicates skill dependence; $\Delta(x)\approx 0$ indicates internalization. Standard GRPO cannot observe this signal because all rollouts share one prompt condition. \textsc{SkillC} exposes it through paired rollouts.

\subsection{Paired Contrastive Rollouts}
\label{sec:paired}

\textsc{SkillC} enforces paired rollouts for every active task type. At each step $t$, the curriculum gate value $\mathrm{gate}^{(t)}(k)$ is converted into a skill-injected rollout fraction:
\begin{equation}
  \rho_\mathrm{with}^{(t)}(k)
  =
  \rho_{\min}
  +
  (1-2\rho_{\min})[2\mathrm{gate}^{(t)}(k)-1]_+.
\end{equation}
This mapping reduces skill injection from dominant early-training levels to a small probing budget $\rho_{\min}$ as internalization progresses.

For each active task $x$, we sample $G$ rollouts and assign $n_+^{(t)}(x)=\mathrm{clip}(\operatorname{round}(G\rho_\mathrm{with}^{(t)}(k(x))),1,G-1)$ to the skill-injected branch, with the remainder to skill-free. Let $\mu^+(x)$ and $\mu^-(x)$ denote the mean rewards of the two branches. The batch-level contrastive signal is
\begin{equation}
  \hat{\Delta}(x)
  =
  \mathrm{clip}\!\Bigl(
    \mathrm{EMA}_{\alpha_{\mathrm{batch}}}\!\bigl(\mu^+(x)-\mu^-(x)\bigr),
    -\delta,\,
    \delta
  \Bigr),
\end{equation}
where $\delta$ is a clipping bound. Unlike utility-based skill scoring or reward shaping, this contrastive signal is used exclusively for credit assignment and curriculum control; it is never converted into skill bonuses or trajectory rewards. For retired task types, $\hat{\Delta}(x)=0$ by convention.

\subsection{Dual-Stream Advantage Estimation}
\label{sec:dual}

\textsc{SkillC} injects this signal via a dual-stream advantage estimator. When skills are beneficial, skill-injected trajectories achieve higher rewards. Standard joint normalization inflates the baseline, suppressing advantages for successful skill-free trajectories precisely when internalization is most critical. To address this, \textsc{SkillC} decomposes the advantage into two complementary streams:
\begin{equation}
  \hat{A}_\mathrm{CSCA}(i,x)
  =
  \Phi(R_i,\mathcal{R}_{\mathrm{batch}})
  +
  \omega \cdot A^\mathrm{cond}_\mathrm{contra}(i,x),
\end{equation}
where Stream 1 preserves global ranking via $\Phi(R_i,\mathcal{R}_{\mathrm{batch}})=(R_i-\mu_{\mathrm{ep}})/(\sigma_{\mathrm{ep}}+\epsilon)$, and the correction magnitude is
\begin{equation}
  C(x)=\lambda_\mathrm{eff}(x)[\hat{\Delta}(x)]_+,
\end{equation}
where $\lambda_\mathrm{eff}(x)=\lambda_\mathrm{dyn}^{(t)}(k(x)) /(1+\beta_\mathrm{freq}n_x)$ scales the signal by internalization progress and task frequency.

Stream 2 performs condition-wise normalization and reallocates credit according to the contrastive signal:
\begin{equation}
  A^\mathrm{cond}_\mathrm{contra}(i,x)
  =
  \begin{cases}
  \displaystyle \frac{R_i-\mu^+(x)}{\sigma^+(x)+\epsilon} - C(x), & z_i=1, \\[8pt]
  \displaystyle \frac{R_i-\mu^-(x)}{\sigma^-(x)+\epsilon} + C(x), & z_i=0.
  \end{cases}
\end{equation}

When $\hat{\Delta}(x)>0$ (skill is still helpful), Stream 2 downweights skill-dependent trajectories (subtracting $C(x)$) while upweighting skill-free successes (adding $C(x)$), encoding the learning objective: "succeed without skill support." This signal is transient by design: as internalization progresses and $\hat{\Delta}(x)$ narrows, $C(x)\to 0$ and Stream 2 degenerates to stratified normalization, a well-established variance-reduction technique. When $\hat{\Delta}(x)\le 0$, skills are no longer helpful and the correction is clipped to zero; the curriculum handles skill retirement. This architecture avoids mixed-normalization bias while maintaining interpretability and computational efficiency.

\subsection{Internalization-Aware Curriculum}
\label{sec:curriculum}

\textsc{SkillC} employs the contrastive signal at two timescales. The advantage estimator uses batch-level $\hat{\Delta}(x)$ for fine-grained per-task credit assignment, while the curriculum uses a validation-level counterpart for coarser-grained withdrawal control. This separation is important because per-batch estimates are noisy and optimized for gradient direction, whereas curriculum decisions require stability.

For each task type $k$, we periodically (every $d$ training steps) compute a validation contrastive gap and smooth it across checkpoints:
\begin{equation}
\begin{aligned}
  \Delta_\mathrm{val}^{(t)}(k)
  &=
  \mathrm{SR}_{\mathrm{with}}^{(t)}(k)
  -
  \mathrm{SR}_{\mathrm{without}}^{(t)}(k), \\
  \bar{\Delta}_\mathrm{val}^{(t)}(k)
  &=
  \alpha_{\mathrm{val}}\bar{\Delta}_\mathrm{val}^{(t-1)}(k)
  +
  (1-\alpha_{\mathrm{val}})\Delta_\mathrm{val}^{(t)}(k), \\
  \mathrm{gate}^{(t)}(k)
  &=
  \sigma\!\left(\bar{\Delta}_\mathrm{val}^{(t)}(k)/T\right).
\end{aligned}
\end{equation}
The gate controls three mechanisms: (i) attribution strength $\lambda_\mathrm{dyn}^{(t)}(k)=\lambda\,\mathrm{gate}^{(t)}(k)$, (ii) rollout fraction $\rho_\mathrm{with}^{(t)}(k)$, and (iii) active-set pruning for multi-category environments:
\begin{equation}
\begin{aligned}
  \mathcal{K}_{\mathrm{keep}}^{(t)}
  &=
  \left\{
    k \in \mathcal{K}_{\mathrm{active}}^{(t-1)}
    :
    \bar{\Delta}_\mathrm{val}^{(t)}(k) \ge \tau_\mathrm{retire}
  \right\}, \\
  \mathcal{K}_{\mathrm{active}}^{(t)}
  &=
  \operatorname{TopK}_{\bar{\Delta}_\mathrm{val}^{(t)}}
  \left(
    \mathcal{K}_{\mathrm{keep}}^{(t)},
    K^{(t)}
  \right).
\end{aligned}
\end{equation}
For single-skill environments, the active set is constant and only mechanisms (i) and (ii) apply. Once a task is retired, it never re-enters, maintaining monotone structure and avoiding oscillatory behavior near zero.

\subsection{Training Objective}
\label{sec:objective}

Training follows standard GRPO with PPO-style clipping and KL regularization, with the vanilla group-normalized advantage replaced by $\hat{A}_\mathrm{CSCA}$:
\begin{equation}
\begin{aligned}
  J(\theta)
  &=
  \mathbb{E}_{x,\{\tau_i\}}\!\left[
    \frac{1}{G}\sum_{i=1}^{G}
    \frac{1}{T_i}\sum_{t=1}^{T_i}
    \ell_{i,t}(\theta)
  \right], \\
  \ell_{i,t}(\theta)
  &=
  \mathcal{L}_\mathrm{clip}\!\left(
    \rho_{i,t},
    \hat{A}_\mathrm{CSCA}(i,x)
  \right)
  -
  \beta\,\mathrm{KL}_{i,t},
\end{aligned}
\end{equation}
where $\rho_{i,t}=\pi_\theta(a_{i,t}\mid h_{i,t})/\pi_{\theta_{\mathrm{old}}}(a_{i,t}\mid h_{i,t})$ and $\mathrm{KL}_{i,t}=D_\mathrm{KL}(\pi_\theta(\cdot\mid h_{i,t})\,\|\,\pi_\mathrm{ref}(\cdot\mid h_{i,t}))$.

\textsc{SkillC} modifies credit assignment within a standard policy gradient framework. It jointly optimizes skill-injected and skill-free trajectories, with the contrastive term redirecting credit toward behaviors that succeed autonomously. However, effectiveness depends critically on signal informativeness: strong contrastive signals enable effective credit redistribution and responsive curriculum adaptation, while weak signals on high-performance tasks can cause premature skill retirement and performance degradation.

\section{Experiments}
\label{sec:experiments}
\begin{table*}[t]
\centering

\setlength{\tabcolsep}{4.7pt}
\renewcommand{\arraystretch}{0.4} 

\begin{tabular}{lccccccccc}
\toprule
 & \multicolumn{7}{c}{\textbf{ALFWorld}} & \multicolumn{2}{c}{\textbf{WebShop}} \\
\cmidrule(lr){2-8} \cmidrule(lr){9-10}
\textbf{Method} & \textit{Pick} & \textit{Look} & \textit{Clean} & \textit{Heat} & \textit{Cool} & \textit{Pick2} & \textit{Avg.} & \textit{Score} & \textit{Succ.} \\
\midrule

\multicolumn{10}{c}{\textit{Training-free Methods}} \\
\midrule
Zero-Shot & 33.4 & 21.6 & 19.3 & 6.9 & 2.8 & 3.2 & 14.8 & 26.4 & 7.8 \\
ReAct \cite{yao2022react} & 48.5 & 35.4 & 34.3 & 13.2 & 18.2 & \cellcolor{secondplace}17.6 & 31.2 & 46.2 & 19.5 \\
Reflexion \cite{shinn2023reflexion} & \cellcolor{firstplace}62.0 & 41.6 & \cellcolor{secondplace}44.9 & 30.9 & \cellcolor{secondplace}36.3 & \cellcolor{firstplace}23.8 & \cellcolor{secondplace}42.7 & \cellcolor{firstplace}58.1 & \cellcolor{firstplace}28.8 \\
Mem0 \cite{chhikara2025mem0} & \cellcolor{secondplace}54.0 & \cellcolor{secondplace}55.0 & 26.9 & \cellcolor{secondplace}36.4 & 20.8 & 7.7 & 33.6 & 23.9 & 2.0 \\
ExpeL \cite{zhao2024expel} & 21.0 & \cellcolor{firstplace}67.0 & \cellcolor{firstplace}55.0 & \cellcolor{firstplace}52.0 & \cellcolor{firstplace}71.0 & 6.0 & \cellcolor{firstplace}46.3 & \cellcolor{secondplace}30.9 & \cellcolor{secondplace}11.2 \\
\midrule

\multicolumn{10}{c}{\textit{Direct RL Methods}} \\
\midrule
PPO \cite{schulman2017proximal} & \cellcolor{secondplace}92.3 & 64.0 & \cellcolor{secondplace}92.5 & \cellcolor{firstplace}89.5 & \cellcolor{secondplace}80.3 & \cellcolor{secondplace}68.8 & \cellcolor{secondplace}80.4 & \cellcolor{secondplace}81.4 & \cellcolor{secondplace}68.7 \\
RLOO \cite{ahmadian2024back} & 87.6 & \cellcolor{secondplace}78.2 & 87.3 & 81.3 & 71.9 & 48.9 & 75.5 & 80.3 & 65.7 \\
GRPO \cite{shao2024deepseekmath} & 90.8 & 66.1 & 89.3 & 74.7 & 72.5 & 64.7 & 77.6 & 79.3 & 66.1 \\
GiGPO \cite{feng2026group} & \cellcolor{firstplace}97.7 & \cellcolor{firstplace}82.7 & \cellcolor{firstplace}98.8 & \cellcolor{secondplace}83.7 & \cellcolor{firstplace}89.3 & \cellcolor{firstplace}79.2 & \cellcolor{firstplace}90.8 & \cellcolor{firstplace}84.4 & \cellcolor{firstplace}72.8 \\
\midrule

\multicolumn{10}{c}{\textit{On-Policy Self-Distillation Methods}} \\
\midrule
OPSD \cite{zhao2026self} & 50.0 & 60.0 & 22.7 & 21.4 & 17.6 & 9.5 & 32.8 & 4.5 & 2.3 \\
GRPO+OPSD \cite{zhao2026self} & 91.4 & 61.5 & \cellcolor{firstplace}100.0 & \cellcolor{secondplace}87.5 & \cellcolor{secondplace}76.5 & 52.2 & 80.4 & 86.8 & 76.5 \\
Skill-SD \cite{wang2026skill} & 93.9 & \cellcolor{firstplace}93.8 & 90.9 & \cellcolor{firstplace}100.0 & 69.2 & \cellcolor{secondplace}68.4 & \cellcolor{secondplace}85.1 & 86.1 & 76.5 \\
RLSD \cite{yang2026self} & \cellcolor{firstplace}100.0 & \cellcolor{secondplace}87.5 & \cellcolor{secondplace}92.3 & 58.8 & \cellcolor{firstplace}80.0 & 65.2 & 82.0 & \cellcolor{secondplace}87.4 & \cellcolor{secondplace}77.3 \\
SDAR \cite{lu2026self} & \cellcolor{secondplace}94.7 & 75.0 & \cellcolor{firstplace}100.0 & 86.7 & 68.2 & \cellcolor{firstplace}78.9 & \cellcolor{firstplace}85.9 & \cellcolor{firstplace}89.4 & \cellcolor{firstplace}82.8 \\
\midrule

\multicolumn{10}{c}{\textit{Skill-Augmented RL Methods}} \\
\midrule
EvolveR \cite{wu2025evolver} & 64.9 & 33.3 & 46.4 & 13.3 & 33.3 & 33.3 & 43.8 & 42.5 & 17.6 \\
\textsc{SkillRL} \cite{xia2026skillrl} & \cellcolor{secondplace}97.9 & 71.4 & 90.0 & 90.0 & 95.5 & 87.5 & 89.9 & 85.2 & 72.7 \\
RetroAgent \cite{zhang2026retroagent} & \cellcolor{secondplace}97.9 & 90.9 & \cellcolor{firstplace}99.2 & 92.9 & 85.3 & \cellcolor{secondplace}91.0 & \cellcolor{secondplace}94.9 & 88.9 & \cellcolor{secondplace}82.3 \\
Skill1 \cite{shi2026skill1} & \cellcolor{firstplace}100.0 & \cellcolor{secondplace}98.6 & \cellcolor{secondplace}97.3 & \cellcolor{firstplace}99.2 & \cellcolor{secondplace}96.1 & \cellcolor{firstplace}96.0 & \cellcolor{firstplace}97.5 & \cellcolor{secondplace}89.7 & \cellcolor{firstplace}82.9 \\
D2Skill \cite{tu2026dynamic} & 97.1 & \cellcolor{firstplace}100.0 & 75.0 & 87.5 & \cellcolor{firstplace}100.0 & 78.6 & 90.6 & \cellcolor{firstplace}91.1 & 80.5 \\
\midrule

\multicolumn{10}{c}{\textit{Skill-Internalization RL Methods}} \\
\midrule
\textsc{Skill0} \cite{lu2026skill0} & \cellcolor{firstplace}100.0 & \cellcolor{firstplace}84.6 & \cellcolor{secondplace}91.1 & \cellcolor{secondplace}84.2 & \cellcolor{secondplace}88.9 & \cellcolor{secondplace}55.5 & \cellcolor{secondplace}85.9 & \cellcolor{secondplace}83.3 & \cellcolor{secondplace}70.9 \\

\textsc{SkillC} &\cellcolor{secondplace} 88.5 & \cellcolor{secondplace}78.6 & \cellcolor{firstplace}91.2 & \cellcolor{firstplace}94.7 & \cellcolor{firstplace}100.0 & \cellcolor{firstplace}88.9 & \cellcolor{firstplace}90.6 & \cellcolor{firstplace}85.6 & \cellcolor{firstplace}74.0 \\
\bottomrule

\end{tabular}
\caption{Main results on ALFWorld and WebShop (Success Rate, \%). Within each category, cells with a yellow background denote the best results, and a grey background denotes the second-best results.}
\label{tab:main_results}
\end{table*}

\subsection{Experimental Setup}
\label{sec:setup}

\paragraph{Environments.}
We evaluate \textsc{SkillC} on two representative agent benchmarks.
ALFWorld \cite{shridhar2020alfworld} is a text-based household environment
aligned with the ALFRED embodied benchmark, where agents complete long-horizon tasks
through multi-step planning and object interaction across six task categories: Pick, Look, Clean, Heat, Cool, and Pick2.
WebShop \cite{yao2022webshop} is a realistic online shopping environment
requiring multi-step search, browsing, and purchase decisions under natural-language instructions.

\paragraph{Baselines.}
We compare \textsc{SkillC} against representative methods from five categories.
Training-free methods include ReAct~\citep{yao2022react},
Reflexion~\citep{shinn2023reflexion}, Mem0~\citep{chhikara2025mem0},
and ExpeL~\citep{zhao2024expel}.
Direct RL methods include PPO~\citep{schulman2017proximal},
RLOO~\citep{ahmadian2024back}, GRPO~\citep{shao2024deepseekmath},
and GiGPO~\citep{feng2026group}.
On-policy self-distillation methods include
OPSD~\citep{zhao2026self},
Skill-SD~\citep{wang2026skill}, RLSD~\citep{yang2026self},
and SDAR~\citep{lu2026self}.
Skill-augmented RL methods, including EvolveR~\citep{wu2025evolver}, 
\textsc{SkillRL}~\citep{xia2026skillrl}, RetroAgent~\citep{zhang2026retroagent}, 
Skill1~\citep{shi2026skill1}, and D2Skill~\citep{tu2026dynamic}, 
all retain external skill access at inference time.
Skill-internalization RL methods, including \textsc{Skill0}~\citep{lu2026skill0} 
and our \textsc{SkillC}, are evaluated \emph{without} external skills at inference time.

\paragraph{Implementation Details and Metrics.}

All experiments use Qwen2.5-7B-Instruct~\cite{qwen2025} as the base model, and the full hyperparameters are in Appendix~\ref{app:hyperparams}.
For ALFWorld, we use the six task-category skills from \textsc{Skill0}: Pick, Look, Clean, Heat, Cool, and Pick2.
For WebShop, we construct a single shopping skill covering search, browsing, and purchase strategies specific to the e-commerce domain. The skills used in both environments are obtained using the skill learning prompts proposed in \textsc{Skill0}\cite{lu2026skill0}.
We report with-skill success rate ($\mathrm{SR}_{\mathrm{with}}$) for skill-augmented methods and without-skill success rate ($\mathrm{SR}_{\mathrm{without}}$) for skill-internalization methods on the test split. For ALFWorld, we additionally track the internalization gap $\Delta_{\mathrm{val}} = \mathrm{SR}_{\mathrm{with}} - \mathrm{SR}_{\mathrm{without}}$ on the validation split to measure progress during training. For WebShop, we report both task score (0--100, measuring attribute relevance) and success rate (\%, binary completion). All reported metrics represent the average of 5 independent runs.


\subsection{Main Results}
\label{sec:main_results}

Table~\ref{tab:main_results} reports test results on ALFWorld and WebShop.
\textsc{SkillC} achieves the strongest \emph{without-skill} performance among all
skill-internalized methods, improving over \textsc{Skill0} from $85.9\%$ to
$90.6\%$ on ALFWorld and from $70.9\%$ to $74.0\%$ on WebShop.
The improvement is attributable to resolving internalization blindness: by making
the task-level contrastive gap $\hat{\Delta}(x)$ observable within each policy
update, CSCA redirects credit toward skill-free success throughout training,
rather than relying on the policy to transfer behaviors passively after
curriculum withdrawal.
On ALFWorld, \textsc{SkillC} also matches D2Skill and outperforms \textsc{SkillRL},
both of which retain external skill access at inference time.
D2Skill also samples paired rollouts within the same update, but uses the
performance gap as a hindsight bonus for skill-injected trajectories under
mixed normalization, reinforcing runtime skill use rather than internalization.
Matching its performance without skill access confirms that redirecting the
same contrastive signal toward skill-free credit assignment is sufficient.

\subsection{Ablation Study}
\label{sec:ablation}

Table~\ref{tab:ablation} evaluates the contribution of \textsc{SkillC}'s main components.

\paragraph{Paired rollouts.}
Without pairing, each task is sampled under only one prompt condition per update,
so the task-level contrastive gap $\hat{\Delta}(x)$ is unavailable.
Performance drops to 87.3\,\%, showing that within-update comparison is the main
source of internalization credit.

\paragraph{Dual-stream normalization.}
Replacing per-condition normalization with a mixed baseline lowers performance to
88.4\,\%. When with-skill and without-skill trajectories share one normalization
reference, high-reward skill-injected rollouts inflate the baseline and suppress
credit to successful skill-free trajectories.

\paragraph{Curriculum.}
Curriculum ablations are consistently worse than the full model, but their effect is
smaller than removing pairing or dual-stream credit assignment. Fixing the curriculum
gate reduces SR to 88.9\,\%, removing pruning gives 89.5\,\%, and fixing
$\rho_{\mathrm{with}}{=}0.5$ gives 89.8\,\%. The worst variant, high init score,
reaches only 87.1\,\%, showing that an overly optimistic prior delays correction.
Overall, the validation signal is most useful when it jointly modulates attribution
strength, rollout allocation, and retirement.

\subsection{Internalization Schedule Analysis}
\label{sec:mechanism}

\begin{table}[!t]
\centering
\renewcommand{\arraystretch}{0.4}
\setlength{\tabcolsep}{4pt}

\begin{tabular}{@{}p{5.6cm} r r@{}}
\toprule
\textbf{Variant} & \textbf{SR} & $\boldsymbol{\Delta}$ \\
\midrule
single-condition updates (no pairing)              & 87.3 & $-3.3$ \\
w/ mixed norm                      & 88.4 & $-2.2$ \\
fixed curriculum gate ($= 0.5$)                    & 88.9 & $-1.7$ \\
w/o active-skill pruning                           & 89.5 & $-1.1$ \\
fixed $\rho_{\mathrm{with}}{=}0.5$                & 89.8 & $-0.8$ \\
high init score ($\bar{\Delta}^{(0)}_{\mathrm{val}}{=}0.3$) & 87.1 & $-3.5$ \\
\midrule
\rowcolor{firstplace}
\textbf{\textsc{SkillC}} & \textbf{90.6} & \textbf{ref.} \\
\bottomrule
\end{tabular}
\caption{%
  Ablation on ALFWorld.
  \textit{SR}: without-skill success rate on the test set.
  $\Delta$: gap relative to full \textsc{SkillC}.%
}
\label{tab:ablation}
\end{table}
\begin{figure*}[t]
\centering
\includegraphics[width=0.95\textwidth]{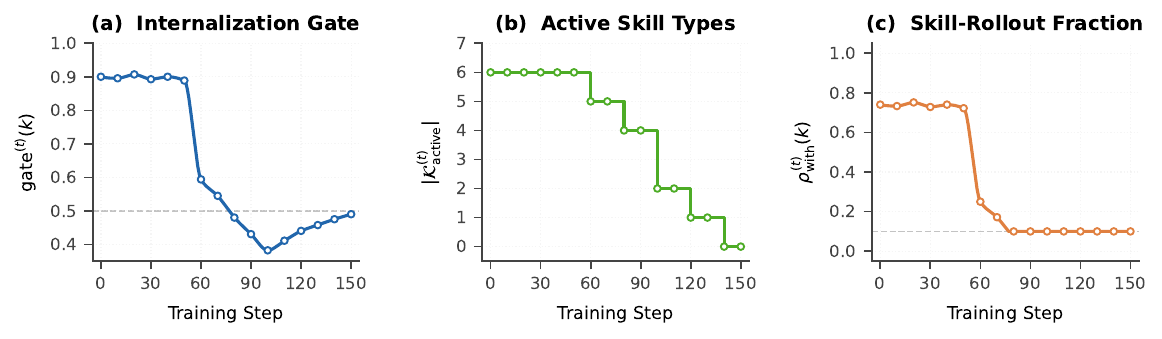}
\caption{%
  Internalization schedule during CSCA training on ALFWorld.
  (a) Internalization gate tracking the smoothed validation gap across all task types.
  (b) Active skill set size showing monotone pruning from six to zero task categories.
  (c) Skill-rollout fraction decreasing as internalization progresses.%
}
\label{fig:csca_dynamics}
\end{figure*}
\begin{figure*}[!t]
\centering
\includegraphics[width=0.95\textwidth]{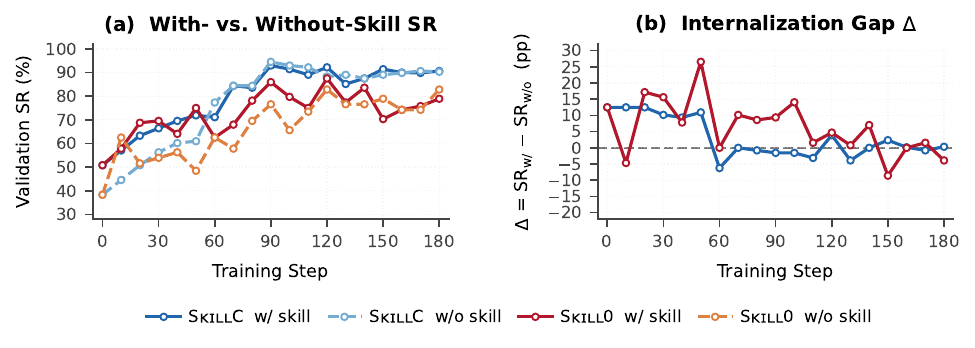}
\caption{%
  Training dynamics on ALFWorld.
  (a) With-skill and without-skill validation success rates over training steps.
  (b) Internalization gap over training steps.%
}
\label{fig:training_dynamics}
\end{figure*}

Figure~\ref{fig:csca_dynamics} traces the internal CSCA signals to confirm that
the gap closure observed in Figure~\ref{fig:training_dynamics} is driven by the
mechanism rather than by incidental factors.
During early training the gate remains near high values, indicating that
skill prompts are still strongly beneficial; at this stage \textsc{SkillC}
follows essentially the same trajectory as \textsc{Skill0}, which rules out
a favorable initialization as the cause of later divergence.
As training progresses, $\bar{\Delta}_{\mathrm{val}}^{(t)}$ declines as the
policy begins to internalize skills, pulling the gate below $0.5$ and
triggering a staggered retirement sequence at irregular intervals that
reflect per-type heterogeneity in internalization speed.
Each retirement immediately reduces the rollout fraction, which is
mechanically tied to the gate, so no manual schedule is required.
By contrast, \textsc{Skill0} withdraws skills on a fixed budget preset before
training; without a live contrastive signal its withdrawal timing is
decoupled from actual internalization progress, and the gap between the
with-skill and without-skill streams persists until the schedule forces
retirement.
The three-way synchronization between the gate, the active set, and the
rollout fraction confirms that \textsc{SkillC}'s internalization schedule
is wholly data-driven rather than externally imposed.

\begin{figure*}[!t]
\centering
\includegraphics[width=0.9\textwidth]{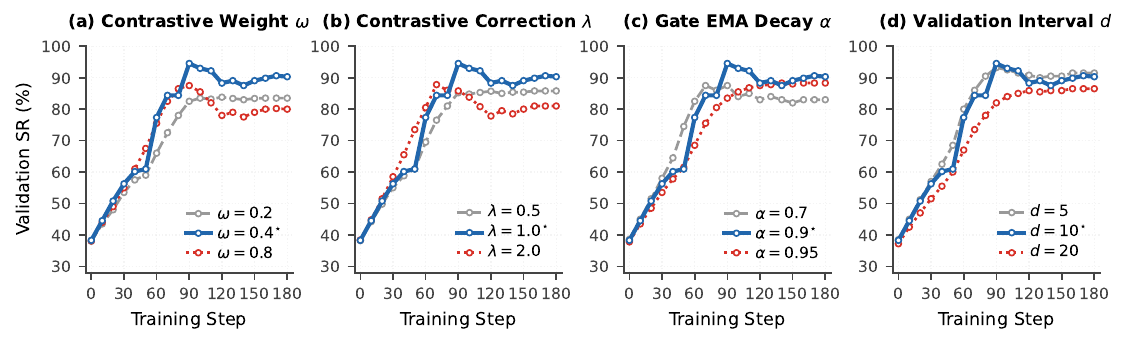}
\caption{Hyperparameter sensitivity on ALFWorld and
  each panel varies one CSCA hyperparameter.}
\label{fig:sensitivity}
\end{figure*}



\subsection{Training Dynamics and Analysis}
\label{sec:analysis}


Figure~\ref{fig:training_dynamics} illustrates the central problem that motivates
\textsc{SkillC}.
For \textsc{Skill0}, the with-skill validation SR consistently exceeds
the without-skill SR by a substantial margin throughout training,
with no sign of spontaneous convergence.
Because \textsc{Skill0} always samples all rollouts under the same prompt condition,
the GRPO objective never observes the contrastive signal $\Delta_{\mathrm{val}}$
and therefore cannot steer the policy toward autonomous success.
The gap at the curriculum switch point is particularly stark: internalization
debt accumulated during the skill-injection phase is exposed all at once when
skills are withdrawn.
This sustained gap is precisely what we term \emph{internalization blindness}.

Under \textsc{SkillC}, the internalization gap starts large at the beginning,
declines steadily through the early-middle phase,
crosses zero during mid-training, and stabilizes slightly below zero in later stages.
This monotonic decline reflects the fact that CSCA provides the policy with a
continuous training signal for autonomous success throughout the entire training
run, rather than concentrating improvement at a discrete curriculum switch.
As shown in Figure~\ref{fig:csca_dynamics}, the internalization gate,
active-set size, and skill-rollout fraction all decrease in lockstep
with the gap, confirming that the internalization schedule responds directly
to measured contrastive evidence rather than a preset timetable.
Once the gap stabilizes below zero, \textsc{SkillC} reaches its best without-skill
validation performance and skill retirement is complete, providing a natural
stopping criterion for internalization-oriented training.

\subsection{Hyperparameter Analysis}
\label{sec:sensitivity}

Figure~\ref{fig:sensitivity} summarizes the effect of four key CSCA hyperparameters while fixing the others at their default values ($\omega{=}0.4$, $\lambda{=}1.0$, $\alpha{=}0.9$, $d{=}10$). The default curve (marked $\star$) is aligned with the \textsc{SkillC} without-skill trajectory in Figure~\ref{fig:training_dynamics} and achieves the highest validation performance. Moving away from this point consistently hurts performance in an interpretable way: smaller $\omega$ or $\lambda$ weakens the contrastive signal and leads to flatter convergence around 85\,\%, whereas larger values improve the early phase but introduce clear late-stage instability, ending near 80\,\%. For the curriculum controller, a smaller $\alpha$ reacts too aggressively to short-term validation noise and drops from an 87.5\,\% peak to about 83\,\%, while a larger $\alpha$ updates too conservatively and tops out around 88.3\,\%. Similarly, more frequent validation ($d{=}5$) stays close to the default but does not surpass it despite higher overhead, whereas less frequent validation ($d{=}20$) slows adaptation and caps performance near 86.5\,\%. Overall, the figure suggests that \textsc{SkillC} is reasonably stable under nearby hyperparameter changes, with the default configuration providing the best trade-off across all four axes.
\begin{table}[!t]
\centering
\renewcommand{\arraystretch}{0.4}
\setlength{\tabcolsep}{2pt}
\begin{tabular}{@{}l r r r r r r@{}}
\toprule
 & \multicolumn{3}{c}{\textbf{Time / Step (s)}} & \multicolumn{3}{c}{\textbf{Active Skills}} \\
\cmidrule(lr){2-4}\cmidrule(lr){5-7}
\textbf{Method} & 50 & 100 & 150 & 50 & 100 & 150 \\
\midrule
GRPO (no skill)  &  941 &  921 &  908 & \multicolumn{3}{c}{---} \\
\textsc{Skill0}    & 1078 & 1061 & 1035 & 6 & 4 & 2 \\
\midrule
\textbf{\textsc{SkillC}}    & 1187 & 1132 & 1095 & 6 & 2 & 0 \\
\bottomrule
\end{tabular}
\caption{Computational cost on ALFWorld. Time/Step is median wall-clock seconds per training step. Active Skills counts remaining task types in the curriculum.}
\label{tab:cost}
\end{table}

\subsection{Computational Cost}

All timings use 8 A100-80GB GPUs with batch size 8 and group size $G=8$. GRPO (skill-free) serves as baseline. \textsc{Skill0} adds static six-skill context, yielding $+14\%$ overhead at step 20. \textsc{SkillC} incurs $+26\%$ early overhead due to paired rollouts, but as skills retire (6 to 0 active skills), the rollout budget decreases and overhead stabilizes at $+20\%$ over \textsc{Skill0}. The increased generation time at later steps reflects improved policy capability rather than algorithmic cost. Total extra compute is approximately $+30\%$ relative to GRPO.
\section{Conclusion}

We address internalization blindness through paired contrastive rollouts and dual-stream advantage estimation, redirecting task-level performance gaps into internalization credit assignment without distorting the optimization target. By reallocating trajectories under condition-wise normalization, we train skill-free success from the start rather than relying on passive transfer after schedule-driven withdrawal. \textsc{SkillC} improves over prior work by 4.7\% on ALFWorld and 3.1\% on WebShop, with full data-driven curriculum control. Future work includes dynamic skill discovery, and generalization to hierarchical skill structures.

\section*{Limitations}

\textsc{SkillC} relies on several assumptions that may limit its applicability:

\paragraph{Skill inventory structure.} The method assumes fixed, task-aligned skill inventories defined before training. In open-ended settings where skill quality or relevance varies dynamically, or where skill-to-task mappings are uncertain, the curriculum gating mechanism may produce unstable retirement decisions.

\paragraph{Validation-level bottleneck.} Reliable internalization progress estimation requires sufficient validation coverage. In environments with extreme class imbalance or rare task types, the smoothed validation signal may become noisy, potentially triggering premature skill retirement.

\paragraph{Computational overhead.} Paired sampling incurs $\sim26\%$ early overhead relative to standard GRPO (Table~\ref{tab:cost}), which may be prohibitive for large-scale training despite the overhead declining as skills retire.
\bibliography{custom}

\appendix

\appendix

\begin{figure*}[t]
\centering
\includegraphics[width=0.95\textwidth]{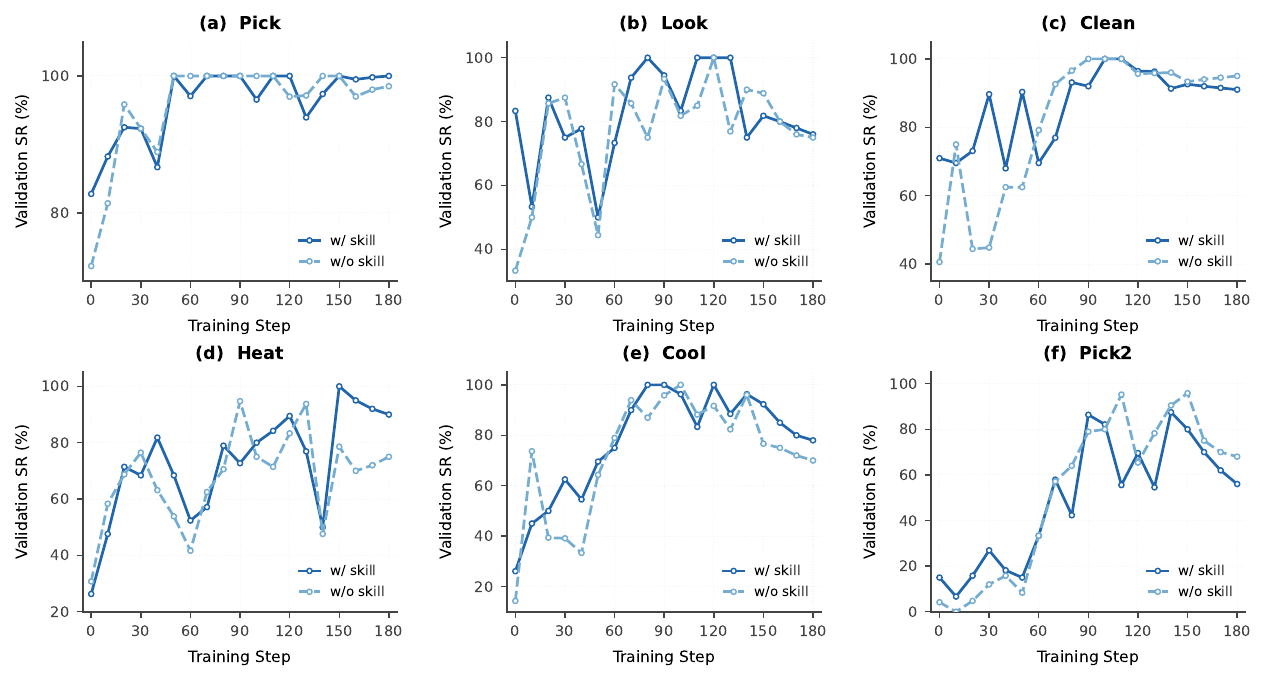}
\caption{Per-task training dynamics across six ALFWorld task categories.
Each subplot shows the with-skill (solid line) and without-skill (dashed line) validation success rates ($\mathrm{SR}_{\mathrm{with}}$ and $\mathrm{SR}_{\mathrm{without}}$) for one task type over 180 training steps.}
\label{fig:per_task_training_dynamics}
\end{figure*}

\section{Implementation Details}
\label{app:hyperparams}
Both ALFWorld and WebShop experiments use the same core CSCA hyperparameters, but differ in environment-specific configuration.

\paragraph{ALFWorld.}
We train on $8{\times}$A100-80\,GB GPUs with batch size~8 tasks/step,
group size $G{=}8$, and learning rate $1{\times}10^{-6}$.
\textsc{SkillC} samples four skill-injected and four skill-free paired rollouts per task,
runs validation every 10 steps (so $d{=}10$), and uses the capacity schedule $K^{(t)}\!\in\!\{6,3,0\}$.
This schedule reflects ALFWorld's six task categories; starting with all six active, then pruning to three midway through training, and finally to zero as skills are fully internalized.

\paragraph{WebShop.}
We use the same GPU configuration, batch size, and learning rate.
\textsc{SkillC} also samples paired rollouts and validates every 10 steps; however, because WebShop has a single unified shopping skill, there is no active-set pruning and $K^{(t)}\equiv 1$ throughout training.
The gate still controls attribution strength and rollout allocation, but these apply uniformly to the single skill category rather than selecting among multiple tasks.

\paragraph{Core CSCA hyperparameters (both environments).}
Key hyperparameters are $\lambda{=}1.0$ (attribution strength), $\omega{=}0.4$ (stream weight), 
$\delta{=}2.0$ (advantage clipping bound),
batch EMA decay $\alpha_{\mathrm{batch}}{=}0.9$, gate temperature $T{=}0.1$,
and validation-gate EMA $\alpha_{\mathrm{val}}{=}0.7$.
The minimal probing budget is $\rho_{\min}{=}0.1$, frequency-penalty weight is $\beta_{\mathrm{freq}}{=}0.01$,
and the retirement threshold is $\tau_{\mathrm{retire}}{=}0.0$.
In early optimization steps, the contrastive gate rises gradually as reliable
paired evidence accumulates, reflecting the contrastive learning signal rather
than a separate initialization trick.

\section{Per-Task Training Dynamics Analysis}
\label{app:per_task_analysis}


Figure~\ref{fig:per_task_training_dynamics} presents a comprehensive breakdown of the training dynamics across all six ALFWorld task categories over the full 180-step training horizon. All curves represent data collected from a single SkillC training run on the Qwen2.5-7B-Instruct base model.

\paragraph{Task-specific internalization patterns.}
The per-task breakdown reveals that internalization speed and baseline performance vary substantially across task categories.
High-performance tasks (Pick, Look, Clean) begin training with with-skill SR already near 80 to 90 percent, reflecting their inherent simplicity or close alignment with the base model's capabilities.
These tasks show rapid convergence and minimal internalization gap throughout the training run.
In contrast, lower-performance tasks (Heat, Cool, Pick2) start below 50 percent with-skill SR and exhibit larger internalization gaps in early training (steps 0 to 50).
For these harder tasks, the training from step 50 to 150 shows clear gap reduction as the policy learns to handle core task logic without skill assistance.

\paragraph{Stability post internalization.}
Figure~\ref{fig:per_task_training_dynamics} demonstrates that once the internalization gap closes (typically by step 100 to 120), both with-skill and without-skill curves stabilize and track closely together.
In the later training regime (steps 150 to 180), performance stabilizes at the learned level for most tasks, with minor fluctuations reflecting natural variance in the validation signal.
This convergence validates the core hypothesis that continuous contrastive feedback enables the policy to acquire autonomous competence gradually, rather than experiencing sudden performance degradation at a curriculum switch.
The fact that without-skill performance remains robust in the later stages (steps 130 to 180) confirms that SkillC has successfully internalized task knowledge rather than merely learning to suppress skill usage.

\section{Theoretical Justification}
\label{app:proofs}

This appendix provides theoretical justification for the dual-stream design of CSCA. We focus on the core asymptotic behavior to show that the contrastive correction term automatically decays as the policy internalizes skills, preventing permanent bias.

\subsection{Asymptotic Correction Decay}

We first formalize the limiting regime in which internalization is complete and the explicit contrastive correction should disappear automatically.

\begin{lemma}[Asymptotic Correction Decay]
\label{lem:decay}

Suppose $\pi_\theta$ has fully internalized skill $\mathcal{S}$, i.e., for all tasks $x$, states $s$, and actions $a$,
\[
P(a \mid s, x, \mathcal{S}; \pi_\theta)
=
P(a \mid s, x; \pi_\theta).
\]
Then $\Delta(x)=0$, and as $\hat{\Delta}(x)\to 0$, the \textsc{SkillC} estimator reduces to
\[
\hat{A}_\mathrm{CSCA}(i,x)
\xrightarrow{\hat{\Delta}\to 0}
A_\mathrm{ep}(i)
+
\omega\cdot
\frac{R_i-\mu^\mathrm{cond}(x)}{\sigma^\mathrm{cond}(x)+\epsilon},
\]
where $A_\mathrm{ep}(i)$ is the global ranking term, and $\mu^\mathrm{cond}(x), \sigma^\mathrm{cond}(x)$ are the within-condition statistics of the branch to which trajectory $i$ belongs.
\end{lemma}

\begin{proof}
If the policy fully internalizes $\mathcal{S}$, then by definition $P(a \mid s, x, \mathcal{S}; \pi_\theta) = P(a \mid s, x; \pi_\theta)$ for all tasks $x$, states $s$, and actions $a$. This implies that the trajectory distributions $\tau^+$ and $\tau^-$ become identical, so $\Delta(x) = \mathbb{E}[R \mid x, z=1] - \mathbb{E}[R \mid x, z=0] = 0$. Since $\hat{\Delta}(x)$ is a smoothed estimate via EMA, we have $\hat{\Delta}(x) \to 0$ as internalization progresses. Therefore, $C(x) = \lambda_\mathrm{eff}(x)[\hat{\Delta}(x)]_+ \to 0$.

By construction, the condition-wise advantage term is 
\[
A^\mathrm{cond}_\mathrm{contra}(i,x) = \frac{R_i - \mu^\mathrm{cond}(x)}{\sigma^\mathrm{cond}(x)+\epsilon} \pm C(x),
\]
where the sign is negative for skill-injected trajectories ($z_i=1$) and positive for skill-free ($z_i=0$). As $C(x) \to 0$, the correction terms $\pm C(x)$ vanish identically, leaving only the within-condition normalization:
\[
A^\mathrm{cond}_\mathrm{contra}(i,x) \to \frac{R_i-\mu^\mathrm{cond}(x)}{\sigma^\mathrm{cond}(x)+\epsilon}.
\]
Therefore, the full CSCA advantage approaches
\[
\hat{A}_\mathrm{CSCA}(i,x) \to A_\mathrm{ep}(i) + \omega \cdot \frac{R_i-\mu^\mathrm{cond}(x)}{\sigma^\mathrm{cond}(x)+\epsilon},
\]
which is the sum of a global ranking term and a stratified within-condition term. This completes the proof.
\end{proof}

\subsection{Design Implications}

The limiting form in Lemma~\ref{lem:decay} has three important implications for the CSCA design.

\paragraph{Automatic correction decay.}
The key insight is that the correction term $C(x) = \lambda_\mathrm{eff}(x)[\hat{\Delta}(x)]_+$ automatically decays to zero as the policy internalizes each skill, eliminating the risk of permanent bias. Unlike methods that apply fixed correction terms or external bonuses, CSCA ensures that its contrastive signal is transient and task-dependent, vanishing precisely when no longer needed.

\paragraph{Stratified advantage structure.}
At convergence, CSCA reduces to a decomposition of global ranking plus per-condition stratification. The stratified term provides variance reduction when the skill-injected and skill-free branches have heterogeneous reward distributions (e.g., $\sigma^+(x) \gg \sigma^{(\text{free})}(x)$ during early skill-injection phases), as it avoids inflating the baseline with high-reward skill-injected trajectories. This stratification is a natural consequence of independent within-condition normalization rather than a manually imposed mechanism.

\paragraph{Consistency with dual stream design.}
The two stream structure: one preserving global ranking, one applying per condition normalization, is justified by the limiting behavior. Stream~1 ensures that the policy maintains a coherent global quality signal across all trajectories, while Stream~2's correction is temporary and task conditioned. At internalization, Stream~2's contribution naturally vanishes, leaving only stratified advantage estimation, which is a well understood variance reduction technique in policy gradient methods.

\end{document}